\documentclass[twoside]{article}

\usepackage{arxiv}

\usepackage[round]{natbib}

\usepackage[utf8]{inputenc} 
\usepackage[T1]{fontenc}    
\usepackage{hyperref}       
\usepackage{url}            
\usepackage{booktabs}       
\usepackage{amsfonts}       
\usepackage{nicefrac}       
\usepackage{microtype}      
\usepackage{xcolor}         

\usepackage{amssymb}
\usepackage{mathtools}
\mathtoolsset{showonlyrefs}
\usepackage{amsthm}
\usepackage{bbm}
\usepackage{caption}
\usepackage{subcaption}
\usepackage[pdftex]{graphicx}

\newtheorem{theorem}{Theorem}[section]

\newcommand{\Probtest}{\mathbb{P}}

\newcommand{\probtest}{p}
\newcommand{\prob}{\widetilde{p}}
\newcommand{\Etest}{\mathbb{E}}

\newcommand{\dec}{x}
\newcommand{\decset}{\mathcal{X}}
\newcommand{\decpareto}{\decset^\alpha(\cov)}

\newcommand{\cov}{\mathbf{z}}
\newcommand{\covsc}{z}
\newcommand{\y}{y}
\newcommand{\yall}{\mathbf{y}}
\newcommand{\yset}{\mathcal{Y}}
\newcommand{\ypareto}{\yset^\alpha(\cov)}
\newcommand{\yI}{y_1}
\newcommand{\yII}{y_2}
\newcommand{\quant}{q}
\newcommand{\score}{r}
\newcommand{\smallprob}{\alpha}

\newcommand{\testsample}[1]{(\dec_{#1}, \y_{#1}, \cov_{#1})}
\newcommand{\weight}[1]{w(\dec_{#1}, \cov_{#1})}

\newcommand{\bound}[1]{#1^{\smallprob}}
\newcommand{\boundFull}[1]{\bound{#1}(\dec, \cov)}
\newcommand{\boundFlex}[2]{\bound{#1}(\dec_{#2}, \cov_{#2})}

\newcommand{\estquant}[1]{\boundFlex{\widehat{\quant}}{#1}}
\newcommand{\truequant}[1]{\boundFlex{\quant}{#1}}
\newcommand{\empquant}{\kappa^{\smallprob}}

\newcommand{\setcal}{\mathcal{D}''}
\newcommand{\settrain}{\mathcal{D}'}
\newcommand{\setdata}{\mathcal{D}}
\newcommand{\setwithtest}{\mathcal{D}_{+}}
\newcommand{\usetwithtest}{\mathcal{S}_{+}}

\newcommand{\event}{e}

\newcommand{\normal}{\mathcal{N}}
\newcommand{\transpose}{\text{T}}
\newcommand{\ind}[1]{\mathbbm{1}_{#1}}

\begin{document}
%

%

\twocolumn[

\aistatstitle{Learning Pareto-Efficient Decisions with Confidence}

\aistatsauthor{ Sofia Ek \And Dave Zachariah \And  Petre Stoica }

\aistatsaddress{Uppsala University, Sweden \And Uppsala University, Sweden \And Uppsala University, Sweden} ]





\begin{abstract}
The paper considers the problem of multi-objective decision support when outcomes are uncertain. We extend the concept of Pareto-efficient decisions to take into account the uncertainty of decision outcomes across varying contexts. This enables quantifying trade-offs between decisions in terms of tail outcomes that are relevant in safety-critical applications. We propose a method for learning efficient decisions with statistical confidence, building on results from the conformal prediction literature. The method adapts to weak or nonexistent context covariate overlap and its statistical guarantees are evaluated using both synthetic and real data. 
\end{abstract}

\section{INTRODUCTION}
In many decision-making tasks, it is important to balance \emph{trade-offs} between several outcome objectives. For instance, a chosen drug can have intended effects but also side effects; a portfolio can be associated with various risks; and a cost reduction in  production may have environmental impacts. These trade-offs can be formally studied by considering an optimization problem with $m$ different objectives:
\begin{alignat}{3}
&\max_{\dec \in \decset} \quad && \big( f_1(\dec), f_2(\dec), \dots, f_m(\dec) \big),
\label{eq:multiobjective}
\end{alignat}
where $\dec$ is a decision variable and $f_k(\dec)$ is the resulting  objective for outcome $k$ that we wish to maximize. A feasible decision $\dec^*$ strictly dominates an alternative decision $\dec$ if it is superior in at least one objective and not inferior in any other objective. It is said to be \emph{Pareto-efficient} if it is not strictly dominated by any other decision \citep{miettinen2012nonlinear, jin2008pareto, debreu1954valuation}. By evaluating the objectives in \eqref{eq:multiobjective} for each Pareto-efficient decision $\dec^*$ we can analyze the trade-offs between different outcomes in a specific decision problem.

We are here interested in problems where the objective functions in \eqref{eq:multiobjective} are \emph{unknown} but we have access to data from past outcomes of decisions. This data is often subject to uncertainty. We therefore consider  outcome $k$ from decision $\dec$ to have a \emph{random reward}  $\y_k$. In this paper we address the problem of learning Pareto-efficient decisions with random rewards from past training data.

One approach is to consider the mean reward (conditioned on the decision) as an objective function: $f_k(\dec) \equiv \Etest[\y_k|\dec]$, which can be approximated by a model learned from data.  \cite{rolf2020balancing} balance two mean objectives that are averaged over randomized decisions in online recommendation systems and data collection in sensitive ecosystems.

However, mean objective functions will not characterize important tail events nor do they represent typical outcomes when the conditional distributions $\probtest(\y_k|\dec)$ are skewed. A substantial portion of outcome rewards may indeed fall below the mean. In safety-critical scenarios, such as clinical decision support with significant negative rewards, mean objective functions are thus inadequate. An additional complication is that an objective function $\widehat{f}_k(\dec)$ learned from data will differ from $f_k(\dec)$ and can thus lead to erroneous decisions. This is aggravated by the fact that past training data differs systematically from future test data obtained after a given intervention. (See below for further details.) The problem is related to that of off-policy learning from logged data for contextual bandit with single objectives  \citep{langford2008epoch,swaminathan2015counterfactual,joachims2018deep,su2020doubly}, but here the past policy is unknown.

In order to tackle these challenges, we seek to endow the concept of Pareto-efficiency with a notion of statistical uncertainty (Section~\ref{sec:problem}). Then we develop a method that learns efficient decisions with confidence from training data (Section~\ref{sec:method}). The proposed method has the following main features:
\begin{itemize}
  \item it considers the lower tail rewards at any specified level,
  \item it provides finite-sample coverage guarantees for the rewards even when the data-generating process is unknown,
  \item and it enables the study of decision trade-offs using a Pareto-frontier that takes reward uncertainties into account.
\end{itemize}
Our proposed method learns models of conditional reward quantiles and leverages results from the conformal prediction literature to identify Pareto-efficient decisions that have statistical validity
\citep{koenker2001quantile, meinshausen2006quantile,vovk2005algorithmic, shafer2008tutorial}. The method is evaluated using both synthetic and real data.

\section{PROBLEM FORMULATION}
\label{sec:problem}

We consider discrete decisions $\dec$ in a space $\decset$. Each decision has $m$ outcome rewards, represented by a vector $\yall$ in a reward space $\yset \subset \mathbb{R}^m$. The number $m$ is typically small, the examples below illustrate cases with $m=2$ rewards. For generality, we consider decisions taken in different contexts described by a $d$-dimensional covariate vector $\cov$ that may affect the outcomes. 

Suppose that for any decision $\dec$, we have a vector $\yall^\alpha(\dec, \cov)$ in reward space that \emph{lower bounds} each reward in $\yall$ across contexts with a  probability of at least $1-\alpha$. Formally, we express this using element-wise inequality as:
\begin{equation}
\Probtest\big(  \:  \yall^\alpha(\dec, \cov) \preceq  \yall \: | \: \dec \:  \big) \; \geq \; 1-\smallprob, \qquad \forall \dec \in \decset.
\label{eq:goal}
\end{equation}
If a nontrivial  bound can be found, we use it to define efficient decisions that take into account the uncertainty of outcomes, including tail events with very low or negative rewards.

We say that decision $\dec^*$ \emph{strictly dominates another decision $\dec$ with confidence $1-\alpha$} if its reward boundary is not inferior,
\begin{equation}
\yall^\alpha(\dec, \cov) \preceq \yall^\alpha(\dec^*, \cov) ,
\end{equation}
\emph{and} that it is superior in at least one outcome $k$ for which $\y^\alpha_k(\dec, \cov) < \y^\alpha_k(\dec^*, \cov)$. The definition above endows the standard notion of Pareto-efficiency with statistical properties. We can now define the object of interest in this paper, $\decpareto$, the set of all \emph{efficient decisions with confidence} $1-\alpha$. Each decision $\dec$ in $\decpareto$ is not dominated by any other decision and the uncertain outcome rewards $\yall$ fall below the boundary $\yall^\alpha(\dec, \cov) $ with a probability no greater than $\alpha$ across contexts. 

Using $\decpareto$ we can construct a corresponding \emph{frontier} 
$\yset^{\smallprob}(\cov)$ in reward space, which enables a user to quantify trade-offs between different efficient decisions. Figure~\ref{fig:Pareto_fronter_leadexample} illustrates such a frontier in an example with a decision space $\decset = \{0, 1, 2, 3, 4\}$ and outcomes rewards $\yall = [\yI, \yII]$. We can see that two of the decisions are strictly dominated with a confidence of 80\%. A user can now analyze the trade-offs between rewards when choosing among the three efficient decisions. (The full setup is described in the numerical experiment section.)

\begin{figure}
\centering
\includegraphics[width=1.0\linewidth]{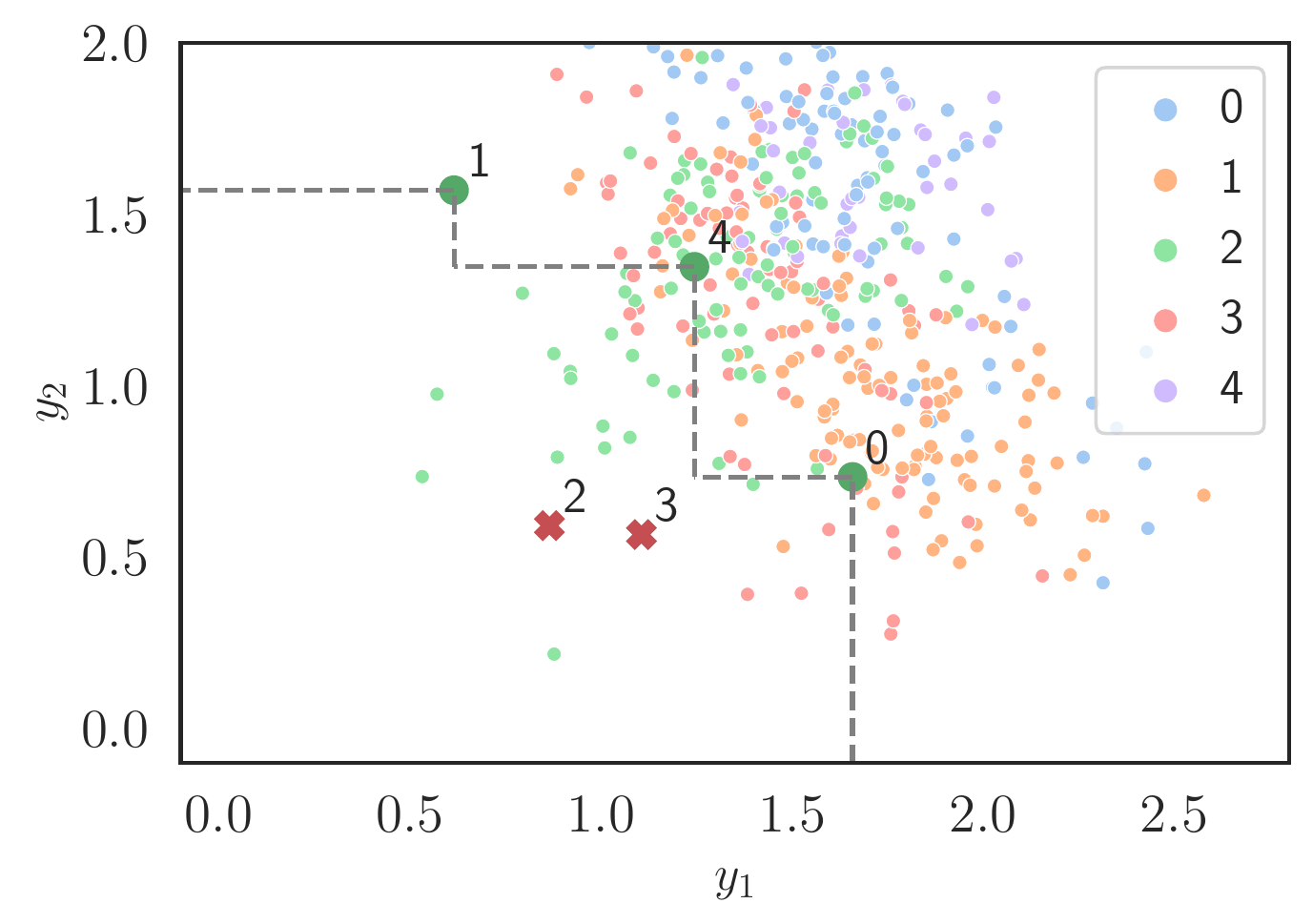} 
\caption{Example of decisions $\dec \in \{ 0,1,2,3,4 \}$ with random outcome rewards $\yall = [\yI, \yII]$. The small dots show past data of rewards observed from decisions taken across different contexts. From this data we learn Pareto-efficient decisions with confidence 80\% across contexts. For a given context specified by $\cov$, this set was $\decpareto = \{ 0, 1,4 \}$ with reward bounds $\yall^\alpha(\dec, \cov)$ shown as solid green dots. Decisions 2 and 3 (red crosses) are thus strictly dominated. The set of efficient decisions and its corresponding reward bounds form a frontier $\yset^{\smallprob}(\cov)$ (dashed) along which a user can study reward trade-offs.}
\label{fig:Pareto_fronter_leadexample}
\end{figure}

The problem we consider in this paper is to \emph{learn} confidence-based efficient decisions $\decpareto$ and their corresponding frontier $\ypareto$ from finite training data. We use past data on decisions, contexts and their rewards $(\dec_i, \yall_i, \cov_i)$ drawn independently and identically from an unknown distribution that admits a causal factorization:
\begin{equation}
\prob(\dec,\yall,\cov) = \probtest(\cov)\probtest(\dec|\cov)\probtest(\yall | \dec, \cov). 
\label{eq:trainingddistribution}
\end{equation}
In the case of past experimental data, the decision policy described by $\probtest(\dec|\cov)$ is known. For instance, in  medical trials, treatment $\dec$ could be assigned at random and independent of context $\cov$. We also consider using past observational data, where $\probtest(\dec|\cov)$ is unknown and has to be learned from data. In this case, one must assume that $\cov$ contains all relevant data and there are no unblocked confounding variables \citep{peters2017elements}. We do \emph{not}, however, need to assume covariate overlap for all decisions. This is an important difference from both causal inference and contextual bandit literatures \citep{imbens2015causal,oberst2020characterization,d2021overlap}.

We shall now proceed to develop a method that learns $\smallprob$-efficient decisions $\decpareto$ from finite training data by constructing nontrivial and context-adaptive bounds $\yall^\alpha(\dec, \cov)$ that satisfy the statistical guarantee \eqref{eq:goal}.

\section{METHOD}
\label{sec:method}


To construct $\decpareto$ we need boundary vectors $\yall^\alpha(\dec,\cov)$ that satisfy
\eqref{eq:goal}. Using the union bound for the complement of the event in  \eqref{eq:goal} we have:
\begin{equation}
\label{eq:union_bound}
\begin{aligned}
&\Probtest\big(  \: \cup^m_{k=1} \y_k < \boundFull{\y_k}    \: | \: \dec \:  \big)
\leq \sum^m_{k=1}  \Probtest( \y_k < \boundFull{\y_k} \: | \: \dec).  
\end{aligned}
\end{equation}
Thus by constructing bounds for each reward that satisfy
\begin{equation}
\Probtest\big(  \y_k < \y^{\smallprob}_k(\dec, \cov) \: | \: \dec  \big)   \leq \frac{\smallprob}{m}, \qquad \forall k,
\label{eq:goal_sub}
\end{equation}
constraint \eqref{eq:goal} is also satisfied.

We consider now the conditional \emph{quantile function} of an individual reward $\y$. For continuous rewards, we can express it as 
\begin{align}
\truequant{} = \inf \left \{ \widetilde{\y} : \frac{\smallprob}{m} = \Probtest(\y < \widetilde{\y} \: | \: \dec, \cov)  \right\}.
\end{align}
Then $\y^{\smallprob}(\dec, \cov) :=\truequant{}$ would provide the tightest lower bound for the reward. But the quantile function is unknown and must be estimated from data
\begin{equation}
\setdata = \big(\testsample{i}\big)^{n}_{i=1},
\end{equation}
drawn from the training distribution \eqref{eq:trainingddistribution}, see, e.g., \cite{koenker2001quantile, meinshausen2006quantile}. However, using an estimate $\y^{\smallprob}(\dec, \cov) := \estquant{}$ in lieu of the unknown function faces two fundamental challenges. First, this choice would not ensure that \eqref{eq:goal} holds for finite $n$. When $\cov$ contains continuous covariates, or is of a moderately large dimension, \eqref{eq:goal} is at best approximately satisfied in the large-sample case provided that we use a well-specified model of the quantile function. Second, the training data is itself a problem because \eqref{eq:trainingddistribution} differs from the distribution that follows when intervening on a \emph{specified} decision variable $\dec := \dec^*$:
\begin{equation}
\probtest(\dec,\yall,\cov) = \probtest(\cov)\mathbbm{1}(\dec=\dec^*)\probtest(\yall | \dec, \cov) ,
\label{eq:testddistribution}
\end{equation}
where $\mathbbm{1}(\cdot)$ is the indicator function.

We will here take a different approach by allowing for \emph{any} flexible machine learning method to fit a quantile function and then form a boundary:
\begin{equation}
    \boxed{\boundFlex{\y}{} := \estquant{} - \empquant(\dec),}
\label{eq:boundary_conformal}
\end{equation}
using an \emph{adjustment} variable $\empquant(\dec)$. By leveraging recent results in the conformal prediction literature \citep{vovk2005algorithmic,lei2015conformal,romanoConformalizedQuantileRegression, tibshirani2019conformal}, we will show that it is possible to learn tight reward boundaries \eqref{eq:boundary_conformal} using flexible methods that ensure the rewards are bounded at the specified probability $1-\smallprob$ in \eqref{eq:goal}. The methodology is based on randomly splitting the data $\setdata$ into two parts, denoted $\settrain$ and $\setcal$.  

Let $\widehat{q}^{\smallprob}(\dec, \cov; \settrain)$ be any fitted quantile function using $\settrain$. Next, define the residuals using the remaining data:
\begin{equation}
  \score_i = \widehat{\quant}^{\smallprob}(\dec_i, \cov_i; \settrain) - \y_i, \qquad i\in \setcal .  
\label{eq:residual}
\end{equation}
Construct a discrete distribution of the residual $\score$:
\begin{equation}
p(\score) = \sum_{i \in \setcal} p_i \ind{}(\score = \score_i) + p_{\infty} \ind{}(\score=\infty),
\label{eq:score_distribution}
\end{equation}
where $\{ p_i \}$ are probability weights.

\begin{theorem}
\label{thm:main}
Consider any $(\dec, \cov)$ drawn from the interventional distribution \eqref{eq:testddistribution} and let the adjustment variable $\empquant(\dec)$ be the $(1 - \frac{\smallprob}{m})$-th quantile of the distribution \eqref{eq:score_distribution} with the following probability weights
\begin{align}
\label{eq:prob_weight_lemma}
p_i &= \frac{\weight{i}}{\sum_{j \in \setcal}\weight{j} + \weight{}}, \\
p_{\infty} &= \frac{\weight{}}{\sum_{j \in \setcal}\weight{j} + \weight{}},
\end{align}
using the function
\begin{equation}
    w(\dec, \cov) = \frac{\mathbbm{1}(\dec=\dec^*)}{\probtest(\dec | \cov)} .
\label{eq:weightfunction}
\end{equation}
Then the boundary \eqref{eq:boundary_conformal} satisfies: 
\begin{equation}
\Probtest ( \y_{} < \y^{\smallprob}(\dec, \cov) | x) \leq \frac{\smallprob}{m}.
\label{eq:theorem}
\end{equation}
Constructing these boundaries for all $m$ rewards therefore satisfies \eqref{eq:goal}.
\end{theorem}

\begin{proof}
The first part of the proof builds on the techniques used Lemma 3 in \cite{tibshirani2019conformal}. We modified it here for the case when $\score$ is only a function of the subset $\settrain$. The second part of the proof builds the techniques used in Theorem 1 and  Corollary 1 in \cite{romanoConformalizedQuantileRegression} and  \cite{tibshirani2019conformal}, respectively.

Let $\setwithtest$ denote the sequence of training data $\setdata$ and a new sample $\testsample{}$ drawn from \eqref{eq:testddistribution}, i.e.,
$$\setwithtest = ( \settrain,  \setcal,\testsample{} ).$$
This sequence has a joint distribution with a density function that can be expressed as
\begin{align}
\prod^n_{i=1} \prob(\dec_i, \y_i, \cov_i) \cdot \prob(\dec, \y, \cov) \weight{}
&= \prob(\setwithtest) \weight{} \\
&= \prob( \usetwithtest ) \weight{},
\end{align}
where $\usetwithtest$ is an unordered set of elements from $\setwithtest$. Note that there are multiple sequences $\setwithtest$ that can be obtained by permuting the elements in $\usetwithtest$.

Let $\event_i$ denote the event that the new sample $\testsample{}$ equals the $i$th sample $\testsample{i}$ in the \emph{sub}sequence
$$\setwithtest'' = ( \setcal, \testsample{} ),$$ 
given the unordered set $\usetwithtest$. The conditional probability of this event can be expressed as
\begin{align}
\Probtest(\event_i|\usetwithtest) 
&= \frac{\Probtest(\event_i , \usetwithtest)}{\sum_{j \in \setcal} \Probtest(\event_j , \usetwithtest)  + \Probtest(\event , \usetwithtest) } \\
&= \frac{\weight{i}}
{\sum_{j \in \setcal}\weight{j} + \weight{}} \\
&= p_i,
\end{align}
where the first equality follows from the law of total probability and the second from considering all permutations $\sigma$ over the set $\usetwithtest$, such that the final element (new sample) in every resulting sequence is equal to $\testsample{i}$. That is,
\begin{align*}
\Probtest(\event_i , \usetwithtest)
&= \sum_{\sigma:\sigma(n+1) = |\settrain| + i} \weight{i}\prob(\usetwithtest) \\
&= \weight{i}\prob(\usetwithtest) n!.
\end{align*}

Each sample $\testsample{i}$  has a residual $\score_i$  of the form \eqref{eq:residual}. The conditional probability that the residual of the new sample $\score$ equals $\score_i$ is equal to
\begin{align*}
\Probtest(\score = \score_i|\usetwithtest) = \Probtest(\event_i | \usetwithtest) = p_i,
\end{align*}
which holds since we condition on the full set $\usetwithtest$ and so $\settrain$ in \eqref{eq:residual} is fixed. Thus the residual $r$ given $\usetwithtest$ follows the discrete distribution:
\begin{align}
\label{eq:lemma_prob}
\sum_{i \in \setcal} p_i \ind{}(\score = \score_i ) + p_{n+1} \ind{} (\score =  \score_{n+1}).
\end{align}
Let $\widetilde{\kappa}^{\smallprob}$ denote its $(1 - \frac{\smallprob}{m})$-th quantile, so that
\begin{equation}
\Probtest \{r > \widetilde{\kappa}^{\smallprob} | \usetwithtest \} \leq \frac{\smallprob}{m}.
\label{eq:inequality_conditioned}
\end{equation}
Now in lieu of \eqref{eq:lemma_prob}, consider \eqref{eq:score_distribution} where the only point mass for the final residual is moved to the limit of its range.  Thus the quantile of \eqref{eq:score_distribution} cannot decrease, $\kappa^{\smallprob} \geq \widetilde{\kappa}^{\smallprob}$, so that
\begin{equation}
\Probtest \{r > {\kappa}^{\smallprob} | \usetwithtest \} \leq \Probtest \{r > \widetilde{\kappa}^{\smallprob} | \usetwithtest \} \leq \frac{\smallprob}{m}.
\end{equation}
Using the fact that $\Probtest \{r > {\kappa}^{\smallprob} | \usetwithtest \} = \Probtest \{y < \widehat{\quant} -  \kappa^{\smallprob} | \usetwithtest \}$, we obtain
\begin{equation}
 \Probtest \{y < \widehat{\quant} -  \kappa^{\smallprob} \}  \leq \frac{\smallprob}{m},
\end{equation}
after marginalization.

Since the interventional distribution \eqref{eq:testddistribution} (and $\weight{})$ contain $\mathbbm{1}(\dec=\dec^*)$, the resulting probability is conditional on any $\dec^* \in \decset$ in \eqref{eq:theorem} since all other decision variables have zero probability. Using \eqref{eq:union_bound} it follows that \eqref{eq:goal} is satisfied.
\end{proof}

\emph{Remark 1:} There is a wide range of possible models and learning methods to fit $\estquant{}$ in \eqref{eq:boundary_conformal}; from simple linear models to highly flexible neural network or random forest-based models, cf. \cite{koenker2001quantile, meinshausen2006quantile}. Whichever method one chooses, the methodology described above ensures that all rewards for each decision will be at least as large as their bounds at a certifiable probability of $1-\smallprob$. The choice of fitted $\estquant{}$ only affects the \emph{tightness} of the bounds.  The efficient decisions $\decpareto$ are thus learned with a model-agnostic confidence of $1-\smallprob$.

\emph{Remark 2:} In the formulation above, we have set an equal probability level of $\smallprob/m$ to each reward. But it is entirely possible to specify different levels across rewards that still satisfy \eqref{eq:goal}. This would be relevant in scenarios where certain rewards are more critical to certify than others. 


\emph{Remark 3:} The conformal prediction literature is focused on the construction of outcome prediction intervals, but we are here only interested in a lower bound on rewards $\yall$. See also  \cite{linusson2014signed}.

\emph{Remark 4:} When the training data is only observational and the decisions are taken with unknown policies the weight function \eqref{eq:weightfunction} must be estimated accurately \citep{leiConformalInferenceCounterfactuals2020}. This can be done by supervised learning of a model of $\probtest(\dec | \cov)$. Alternatively, the weights can be expressed as:
\begin{align}
\weight{}  \equiv  \frac{\ind{}(\dec = \dec^*) }{\probtest(\cov | \dec)\probtest(\dec)} \sum_{\dec' \in \decset }\probtest(\cov|\dec')\probtest(\dec'),
\label{eq:weightfunction_alternative}
\end{align}
where a model of $\probtest(\cov | \dec)$ can be learned in an unsupervised manner.

\section{NUMERICAL EXPERIMENTS}
\label{sec:numerical}
To illustrate the key concepts of our method for learning $\decpareto$ and verify the statistical properties of the frontier $\ypareto$, we conduct numerical experiments with both synthetic and real-world data. For convenience of illustration, we will only consider $m=2$ rewards. We split the training data $\setdata$ in two parts $\settrain$ and $\setcal$ of equal size.  

The conformal inference part of our implementation is based on the code from \cite{romanoConformalizedQuantileRegression} which we extended to account for the distributional shift that arises between training and interventional distributions \eqref{eq:trainingddistribution} and \eqref{eq:testddistribution}.

\subsection{Synthetic data}
\label{sec:synthetic_data}

We consider a problem with five decision alternatives $\decset = \{0, 1, 2, 3, 4\}$ and a one-dimensional covariate $z$. We begin by specifying the unknown training distribution \eqref{eq:trainingddistribution}.

The rewards are drawn from $\probtest(\yall|\dec, z)$ as:
\begin{align}
    \begin{cases}
    \yI &= a_{\dec} + b_{\dec} f \left (\frac{\covsc - 55}{9} \right)   + u_{0},  \\
    \yII &= c_{\dec} + d_{\dec} f \left (\frac{\covsc - 50}{8} \right)   + u_{1},
    \end{cases}
    \label{eq:rewards}
\end{align}
where $f(\cdot)$ is the sigmoid function and the noise terms $u_0$ and $u_1$ are drawn jointly from $\mathcal{N}(0,0.04 \cdot [\begin{smallmatrix}
    1 & \rho \\ 
    \rho & 1
\end{smallmatrix}])$ with a correlation coefficient $\rho = -0.2$ unless another value is explicitly mentioned. (The rewards are ensured to be nonnegative by truncation in the rare cases.) The coefficients in \eqref{eq:rewards} are given in Table~\ref{tab:constants}. The covariate is drawn from $\probtest(\covsc)$ as $\covsc \sim \normal(60, 100)$.

\begin{table}[htbp]
\caption{The constants in \eqref{eq:rewards}.}
\begin{center}
\begin{tabular}{|c|c|c|c|c|c|}
\hline
 & $\dec = 0$ & $\dec = 1$ & $\dec = 2$ & $\dec = 3$ & $\dec = 4$  \\
 \hline
$a_{\dec}$ & 2.4 & 0.7 & 0.8 & 2.0 & 1.2  \\
\hline
$b_{\dec}$ & -1.4 & 1.5 & 1.0 & -1.2 & 1.0 \\
\hline
$c_{\dec}$ & 0.0 & 2.2 & 0.6 & 0.0 & 2.2 \\
\hline
$d_{\dec}$ & 2.4 & -1.5 & 1.0 & 2.0 & -1.0 \\
\hline
\end{tabular}
\label{tab:constants}
\end{center}
\end{table}

In the training data, decisions have been made in different contexts described by $\covsc$. We let the decision assignment be as follows
\begin{align} x = 
    \begin{cases}
    0 \quad \text{if} \quad 0 \leq s < 0.2, \\
    1 \quad \text{if} \quad 0.2 \leq s < 0.4, \\
    2 \quad \text{if} \quad 0.4 \leq s < 0.6, \\
    3 \quad \text{if} \quad 0.6 \leq s < 0.8, \\
    4 \quad \text{if} \quad 0.8 \leq s < 1.0,
    \end{cases}
\end{align}
where $s$ is a random variable. In the case of experimental data, we consider uniform random assignments: $s \sim \mathcal{U} (0, 1)$ and take $\probtest(\dec| \covsc)$ to be known. In the case of observational data, we test the limits of our method by considering an unbalanced scenario where some of the decisions are rare so that the unknown $\probtest(\dec|\covsc)$ is near zero (thus yielding weak covariate overlap). This is done by setting
\begin{align}
    s = u \times f\left (\frac{70 - \covsc}{5}\right ), \text{ where } u \sim \mathcal{U} (0, 1).
\end{align}
We draw $n = 1000$ training data points to learn $\decset^{\smallprob}(\covsc)$ with confidence level $1-\smallprob$ using a quantile random forest method  \citep{meinshausen2006quantile}. We verify that the method satisfies \eqref{eq:goal} at the specified confidence level in the experimental data case in Fig.~\ref{fig:coverage_experimental}.

In the observational data case, we learn a (two-component) Gaussian mixture model of $\probtest(\covsc|\dec)$ in \eqref{eq:weightfunction_alternative}. Then \eqref{eq:goal} will only hold approximately depending on the accuracy of \eqref{eq:weightfunction}, see \cite{leiConformalInferenceCounterfactuals2020}. The accuracy of probability weights can be assessed using model validation methods. In Fig.~\ref{fig:coverage_observational} we evaluate bounds obtained with approximate weights and find \eqref{eq:goal} to be satisfied in this challenging setting with poor data overlap.
In Fig.~\ref{fig:coverage_observational_rho} we also include the case where the outcomes are even more correlated by setting $\rho = -1.0$, and observe only minor differences. In Fig.~\ref{fig:coverage_observational_per_x}, we see that due to the uneven past decision policy, data on decision $\dec = 4$ is sparse and the resulting bound becomes notably more conservative than for the other decisions. The bounds for $\dec=0$ and 1 are the least conservative; the probabilities of rewards exceeding them are very close to $1-\smallprob$.


We illustrate the resulting learned frontier $\yset^{\smallprob}(\covsc)$ for $1-\smallprob = 80\%$ and two different contextual covariates $\covsc=56$ and $\covsc=68$ in Figs.~\ref{fig:4.1.1a} and \ref{fig:4.1.1b}, respectively. The learned efficient decisions for the two different  contexts are:
$\decset^{\smallprob}(56) = \{0,4\}$ and $\decset^{\smallprob}(68) = \{ 0,1,2\}$ and thus are context-dependent. (Figure~\ref{fig:Pareto_fronter_leadexample} shows $\covsc=46$.) The method adapts to the lack of data for decision $\dec = 4$ by setting the lower bound to the minimum possible rewards, represented by a black star in  Fig.~\ref{fig:4.1.1b}.

\begin{figure*}
    \centering
    \begin{subfigure}[b]{0.49\textwidth}
        \centering
        \includegraphics[width=\textwidth]{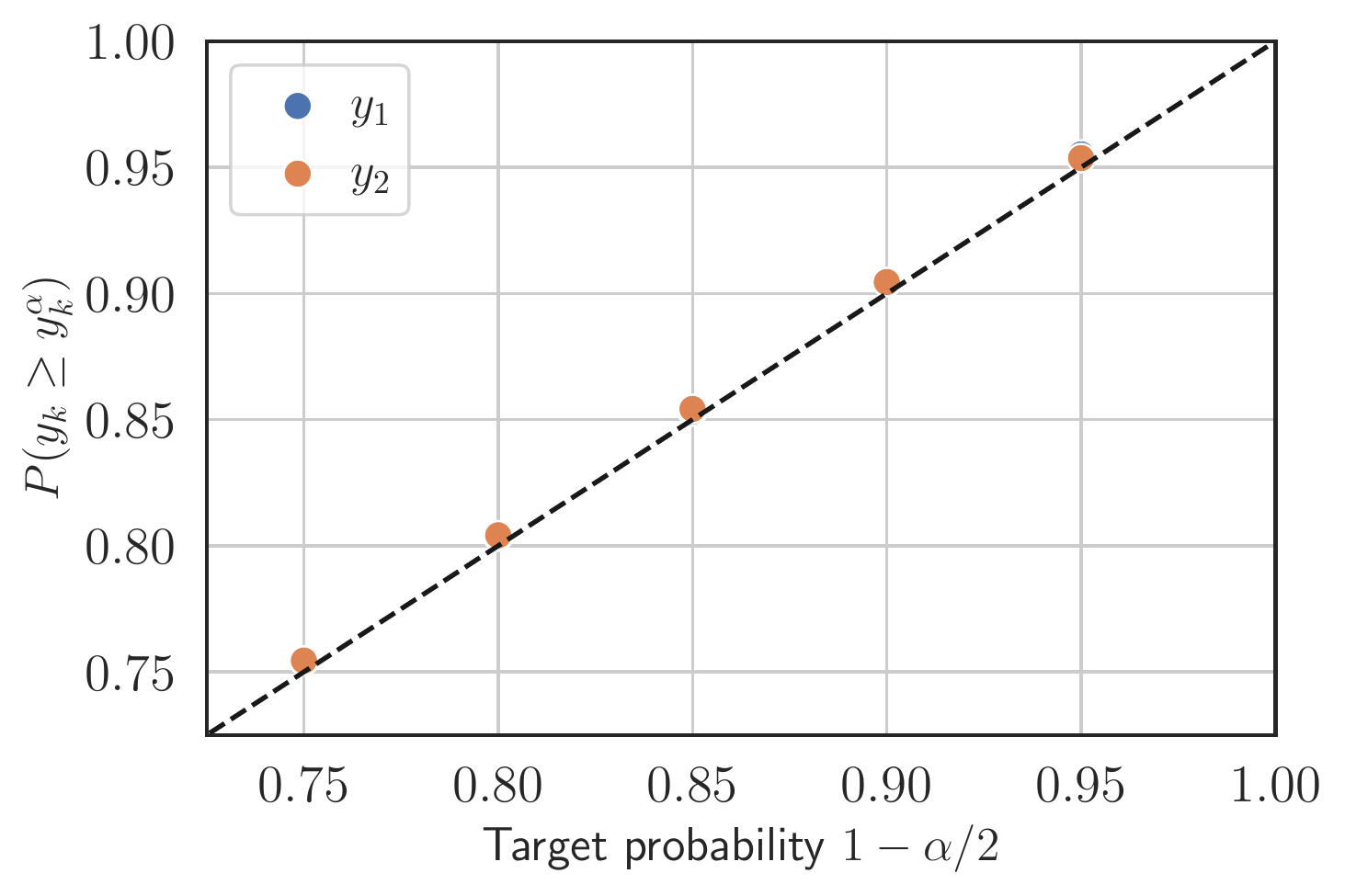}
        \caption{}
        \label{fig:synthetic_single_known}
    \end{subfigure}
    \begin{subfigure}[b]{0.49\textwidth}
        \centering
        \includegraphics[width=\textwidth]{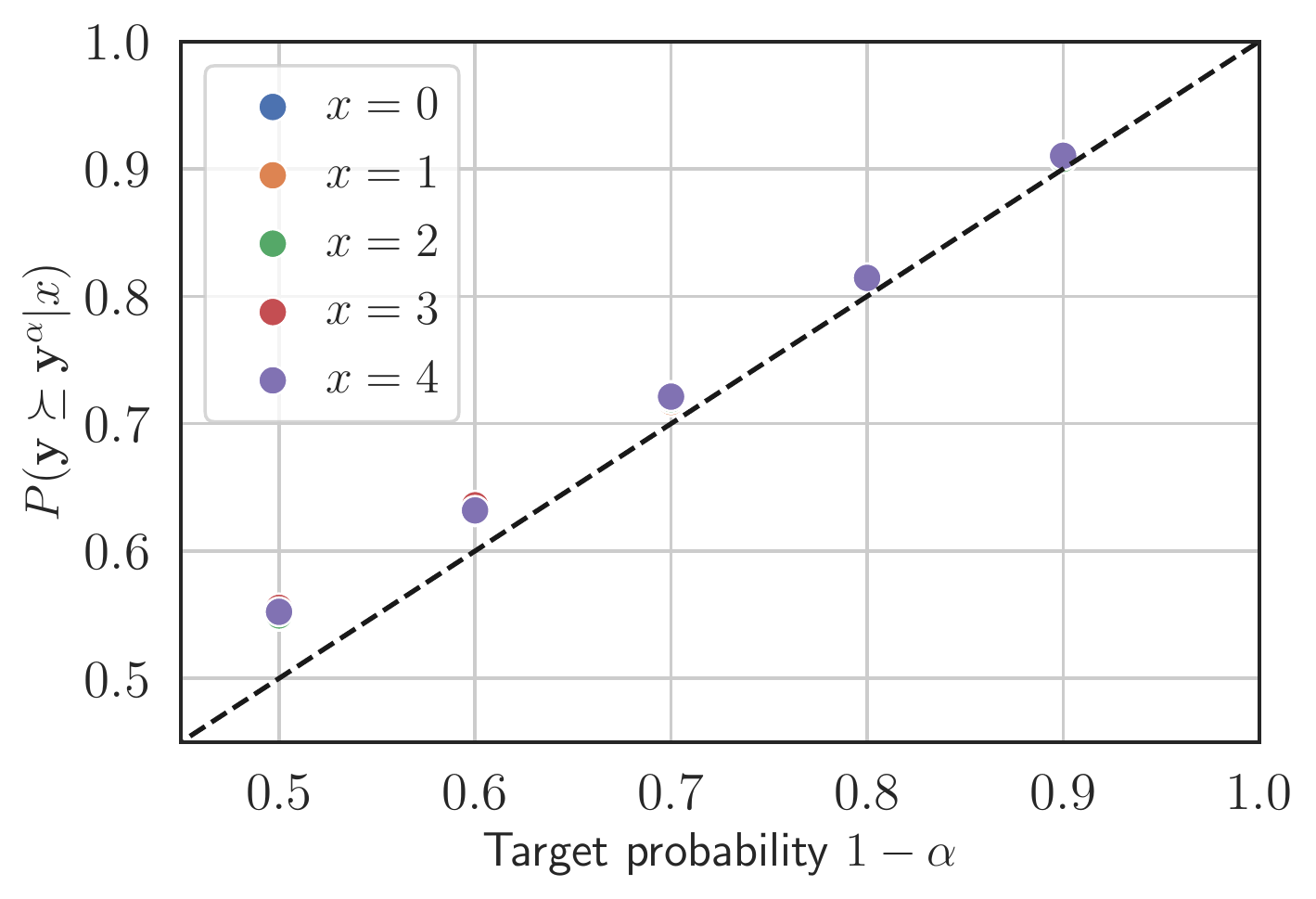}
        \caption{}
         \label{fig:synthetic_union_known}
    \end{subfigure}
    \caption{Evaluation of reward boundary $\yall^{\smallprob}(\dec, \covsc)$ of random rewards $\yall$ using 500 Monte Carlo simulations. Boundaries are designed to satisfy \eqref{eq:goal}. Experimental data scenario with known random assignment probabilities $\probtest(\dec)$. (a) Probability of $\y_k$ exceeding the boundary $\bound{\y_k}(\dec, \covsc)$ marginalized across decisions. (b)  Probability of $\yall$ exceeding the reward boundary $\yall^\alpha(\dec, \covsc)$ under decision $\dec$. Note that the dots are  overlapping in both figures.}
\label{fig:coverage_experimental}
\end{figure*}

\begin{figure*}
    \centering
    \begin{subfigure}[b]{0.49\textwidth}
        \centering
        \includegraphics[width=\textwidth]{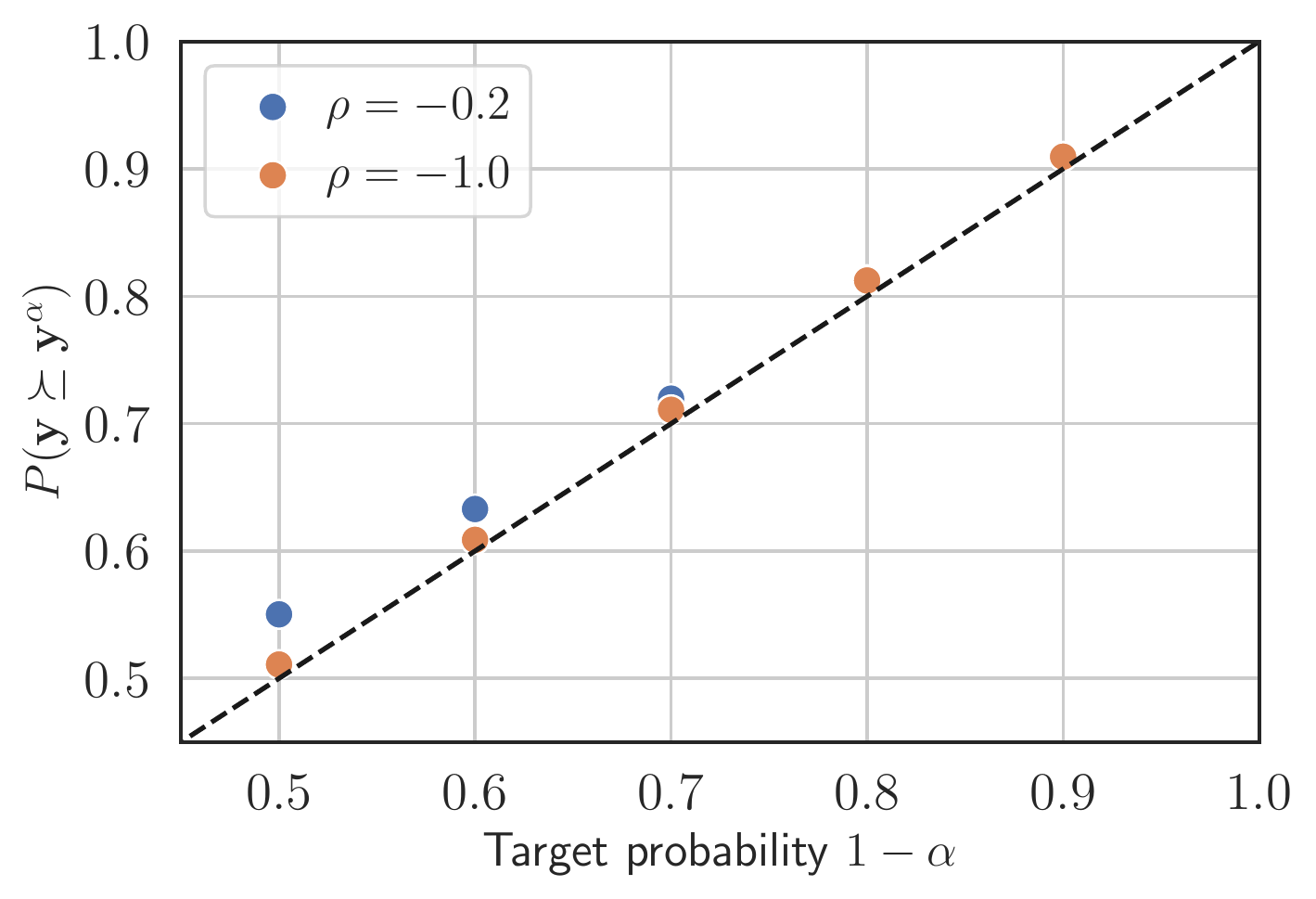}
        \caption{}
        \label{fig:coverage_observational_rho}
    \end{subfigure}
    \begin{subfigure}[b]{0.49\textwidth}
        \centering
        \includegraphics[width=\textwidth]{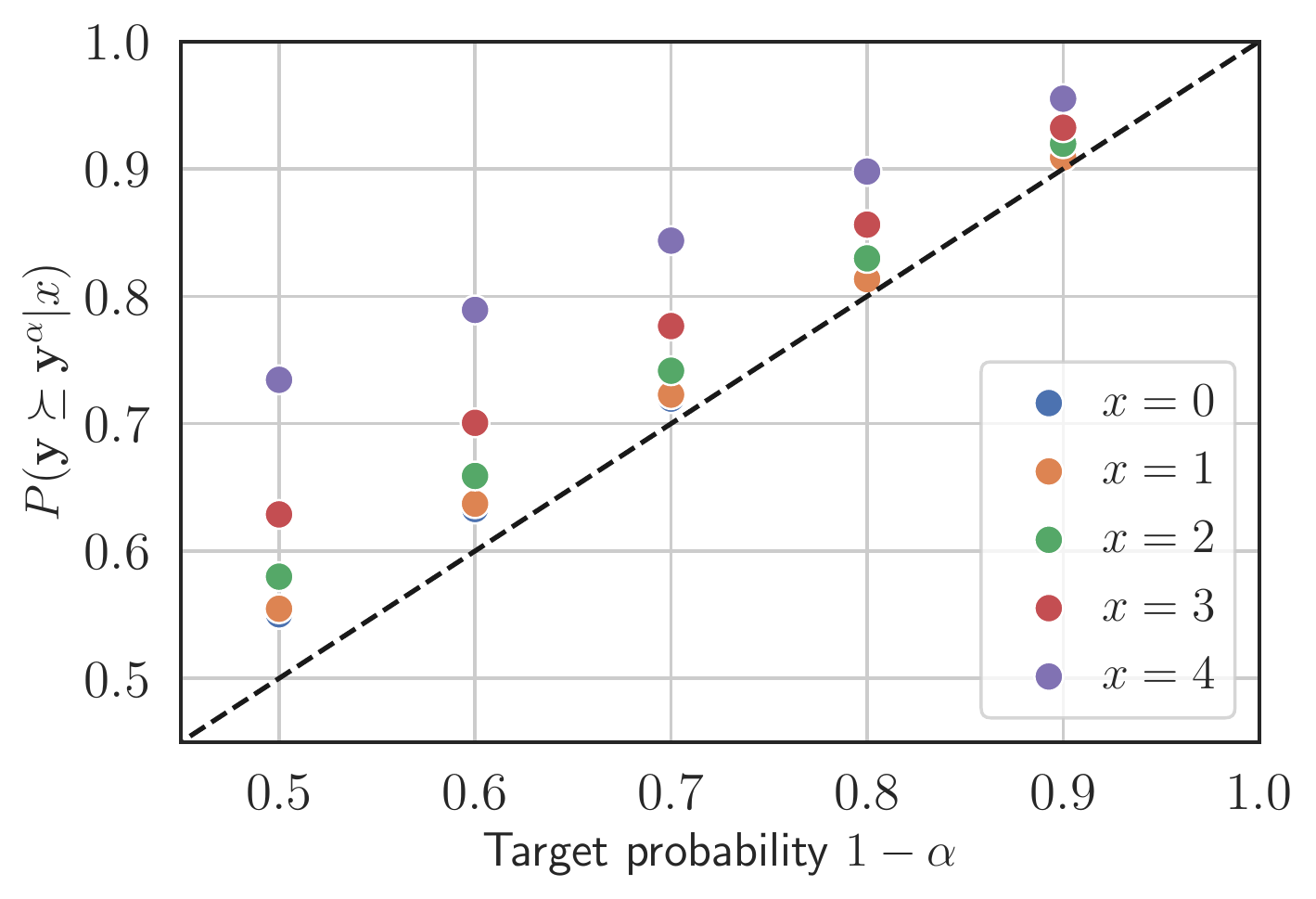}
        \caption{}
         \label{fig:coverage_observational_per_x}
    \end{subfigure}
    \caption{Evaluation of reward boundary $\yall^{\smallprob}(\dec, \covsc)$ of random rewards $\yall$ using 500 Monte Carlo simulations. Observational data scenario where $\probtest(\covsc|\dec)$ in \eqref{eq:weightfunction_alternative} is estimated. (a) Probability of rewards exceeding the boundaries for two values of $\rho$ in \eqref{eq:rewards}, marginalized across decisions. (b)  Probability of $\yall$ exceeding the reward boundary $\yall^\alpha(\dec, \covsc)$ under decision $\dec$.  Note that several dots overlap in both figures.}
\label{fig:coverage_observational}
\end{figure*}

\begin{figure*}
    \centering
    \begin{subfigure}[b]{0.45\textwidth}
        \centering
        \includegraphics[width=1\textwidth]{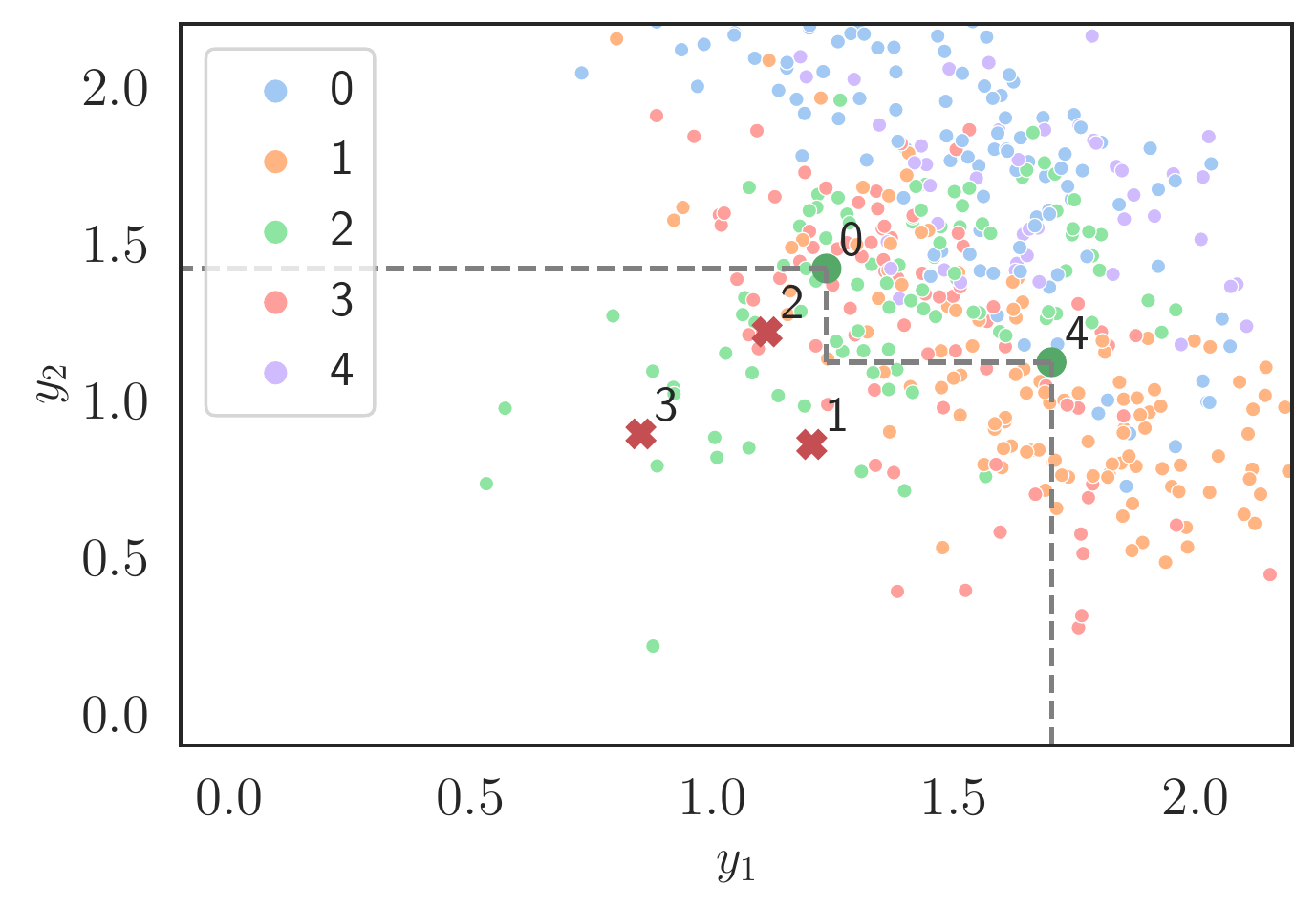} 
        \caption{$z = 56$.}
        \label{fig:4.1.1a}
    \end{subfigure}
    \hfill
    \begin{subfigure}[b]{0.45\textwidth}
        \centering
        \includegraphics[width=1\textwidth]{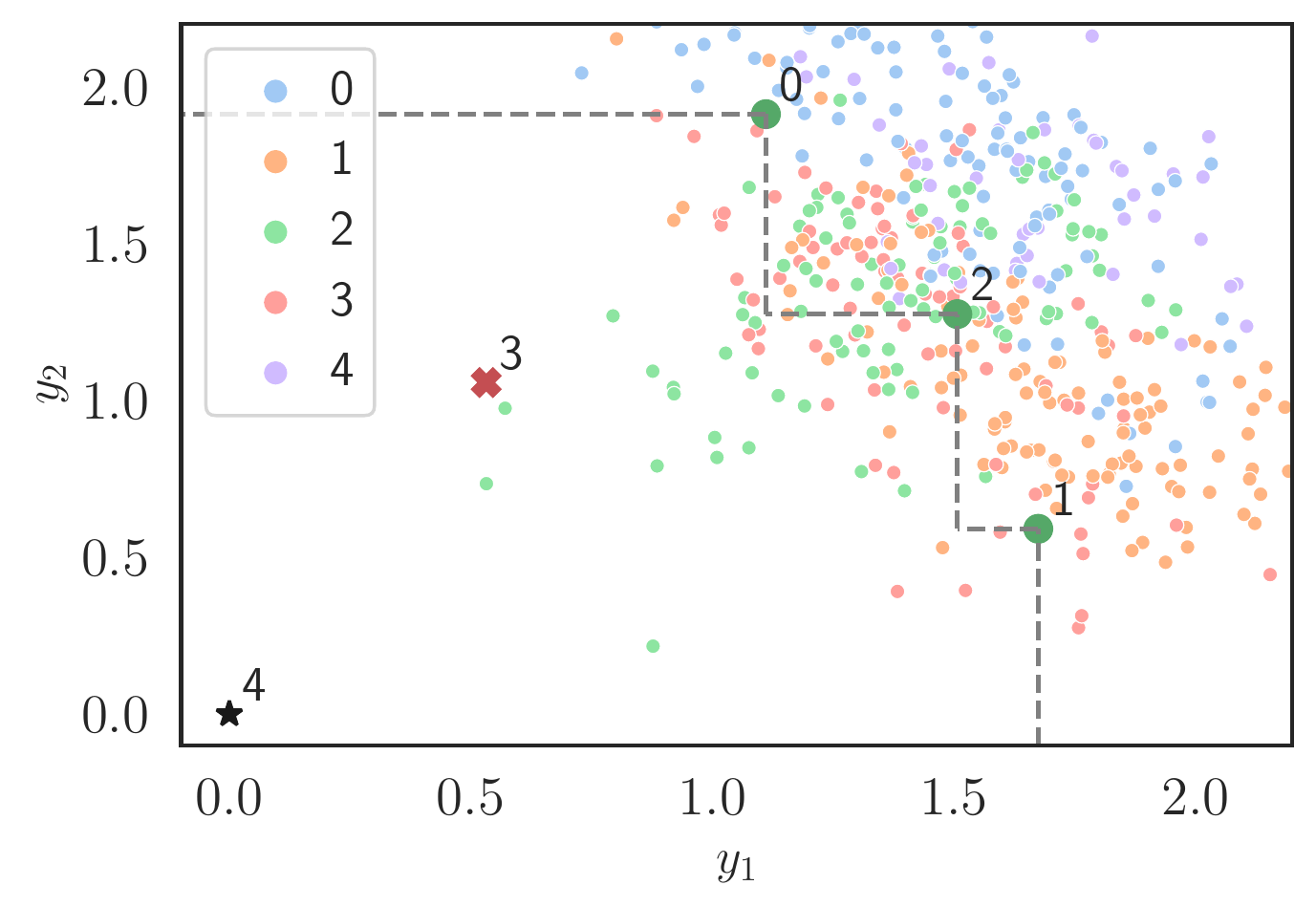} 
        \caption{$z = 68$.}
    \label{fig:4.1.1b}
    \end{subfigure}
    \caption{Efficient decisions with confidence 80\% in the observational data scenario for two different contextual covariates (a) and (b). In (a) the efficient set is $\decset^{\smallprob}(\covsc) = \{ 0, 4 \}$ and decisions 1, 2 and 3 are thus strictly dominated, while in (b) efficient set is $\decset^{\smallprob}(\covsc) = \{ 0, 1, 2 \}$ and decisions 3 and 4 are strictly dominated. In (b) there was no data for decision $\dec = 4$ and the method adapts by setting the minimum possible reward as a boundary (highlighted as a black star). The dashed lines illustrate the frontiers $\yset^{\smallprob}(\covsc)$ for each context.}
    \label{fig:4.1.1}
\end{figure*}

\subsection{Real data}
\label{sec:real_data}
Next, we consider real-world data from the Tennessee Student/Teacher Achievement Ratio (STAR) study \citep{DVN/SIWH9F_2008, krueger1999experimental}. In a four-year class size study teachers and students (in kindergarten up until third grade) were randomly assigned into one of three interventions: small, regular and regular-with-aide class size. Small class size corresponds to 13 to 17 students per teacher, regular class size to 22 to 25 students per teacher and regular-with-aide class size to 22 to 25 students per teacher plus a full-time teacher assistant.  The study was thus a randomized controlled trial, with interventions assigned by a given random policy $\probtest(\dec)$.

We follow the procedure in \cite{kallus2018removing} and take the class-type at first grade as the treatment, since many students were not part of the study in kindergarten. We use 11 pre-intervention covariates in $\cov$ for each student including: gender, birth month, teacher career, teacher experience and if free lunch was given or not. A more detailed description of the covariates is provided in the supplementary material. 

Outcome $\yI$ is an achievement test score at the end of first grade (the sum of standardized math, reading and listening test scores). Students that were not part of the STAR study in first grade or had missing outcome $\yI$ are excluded, and the remaining data set includes 6322 students from 75 different schools. To study a second outcome $\yII$, the (negative) cost per student, we generate a synthetic value following \cite{hill2011bayesian}
\begin{align}
    \yII | \dec &= 0 \sim \normal (-(\cov^{\transpose} \beta - \omega_0 + \mu), \sigma^2), \\
    \yII | \dec &= 1 \sim \normal (-(\exp[(\cov + 0.5)^{\transpose} \beta] + \mu), \sigma^2), \\
    \yII | \dec &= 2 \sim \normal (-(\cov^{\transpose} \beta - \omega_2 + \mu), \sigma^2),
\end{align}
where $\mu = 10$ and $\sigma = 1$. The parameter $\beta$ is drawn randomly from $(0.0, 0.1, 0.2, 0.3, 0.4)$ with probabilities $(0.4, 0.15, 0.15, 0.15, 0.15)$. These probabilities are changed compared to \cite{hill2011bayesian}, since our model has fewer covariates. We select $\omega_0$ and $\omega_2$ so that the effect of the `treatment on the treated' is 4 and 2, respectively.   

The conditional quantiles are estimated with a quantile neural net \citep{taylor2000quantile, steinwart2011estimating}. 



The learned Pareto-frontier for a group of students with a given set of covariates is shown in Fig.~\ref{fig:4.2.1a}. Examples of two such contextual covariates are school in rural areas and teachers with one year teaching experience. Here no decision is dominated at the confidence level 80\%, but the trade-offs between test scores and costs become transparent: small class sizes yield the best test scores but are the most costly interventions, while regular class sizes with aide increase test scores by a lesser amount but at a substantially lower cost. A learned Pareto-frontier for a different group of students is shown in Fig.~\ref{fig:4.2.1b}. Their covariates differ for instance by school located in suburban areas and teachers with 10 years teaching experience. Here the decision 'regular-with-aide' class size is dominated by the other two decisions. 

In the supplementary material, the statistical guarantee for the boundary of the synthetic reward $\yII$ is validated.

\begin{figure*}
    \centering
    \begin{subfigure}[b]{0.45\textwidth}
        \centering
        \includegraphics[width=1\linewidth]{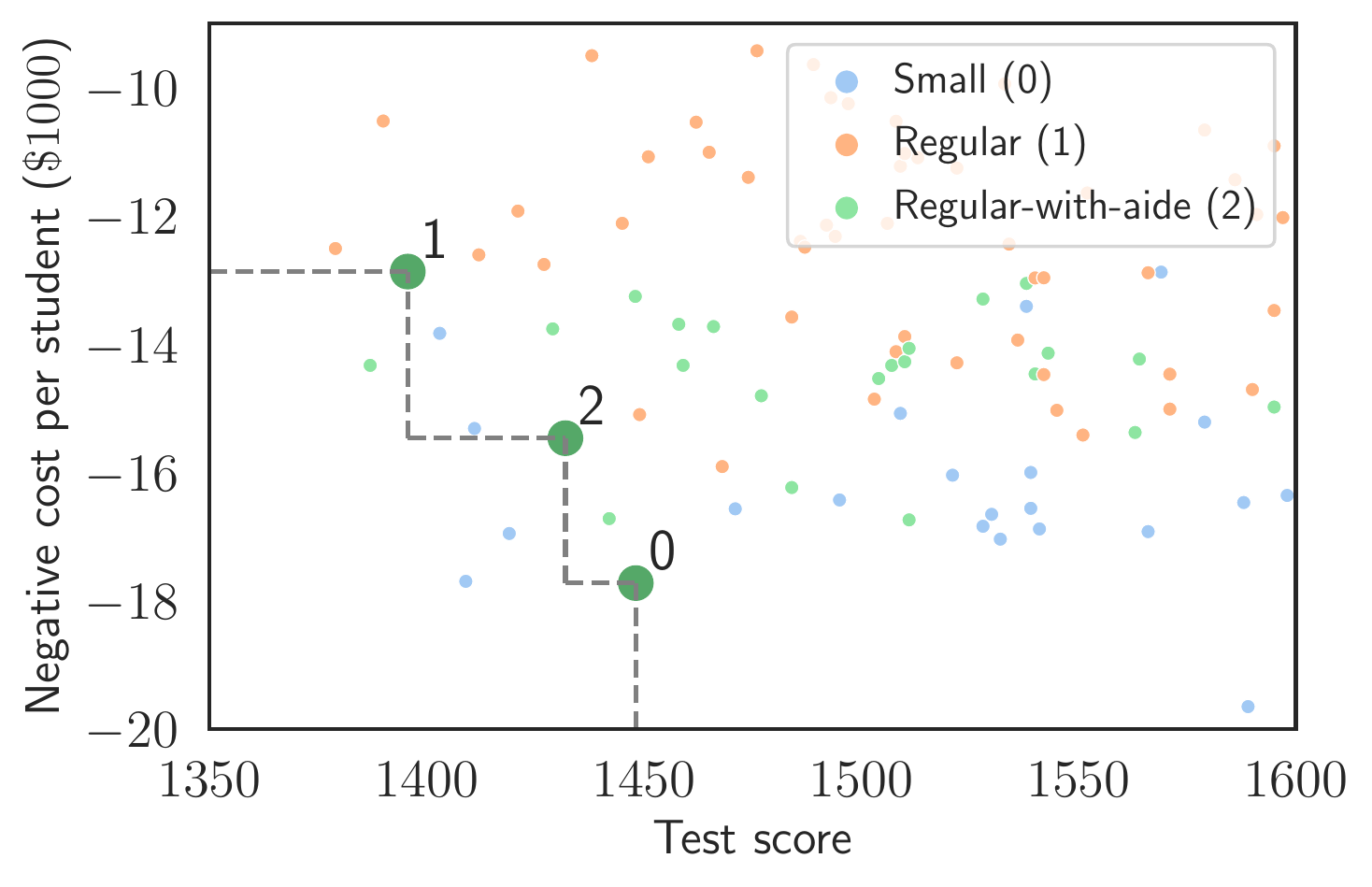} 
        \caption{$\cov'$}
        \label{fig:4.2.1a}
    \end{subfigure}
    \hfill
    \begin{subfigure}[b]{0.45\textwidth}
        \centering
        \includegraphics[width=1\textwidth]{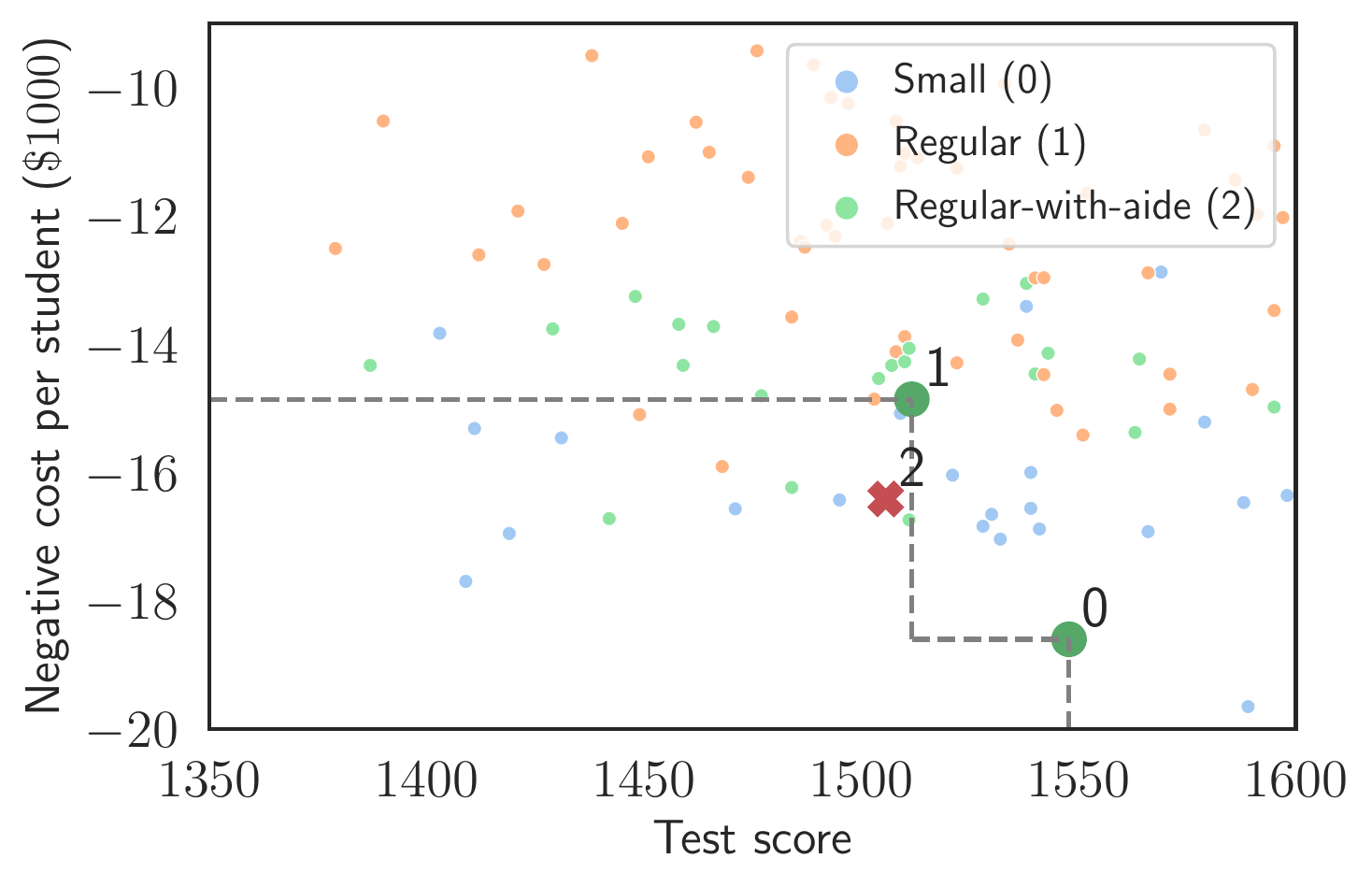}
        \caption{$\cov''$}
        \label{fig:4.2.1b}
    \end{subfigure}
    \caption{Frontiers $\ypareto$ for student groups with two different sets of covariates learned from real data with three different class size interventions. (a) All decisions are efficient so that $\decset^\alpha(\cov')=\{  0,1,2\}$. (b) Efficient decisions are $\decset^\alpha(\cov'')=\{ 0,1\}$ since intervention 2 is dominated by 1. Confidence level 80\%.}
\end{figure*}

\section{DISCUSSION}

Machine learning methods are typically designed to maximize a single objective. When used for decision making, they can therefore result in unbalanced outcomes that do not align with the objectives of a decision-maker \citep{christian2020alignment}. In binary classification tasks, for example, a user may need to balance false-positive and false-negative rates \citep{tong2018neyman} and in recommendation systems is the balance between welfare and profit important \citep{rolf2020balancing}.

In safety-critical applications, such as clinical decision support, it is also important to account for low (even negative) rewards and provide statistical guarantees. In this paper we extended the concept of Pareto-efficiency to include statistical confidence. We then proposed a method that learns efficient decisions and a frontier that takes into account the lower tail rewards at any specified confidence level. Such a frontier can quantify trade-offs and  provide valuable insight to decision makers due to its finite-sample statistical guarantees.

In the case when training data is observational rather than experimental, the (unknown) past decision policy must also be learned. Here care must be taken to assess the accuracy of any  policy model using validation methods; inaccurate models invalidate the statistical properties of the learned frontier and may lead to erroneous decisions. It is also important to consider systematic differences in the distribution of contextual covariates during training and intervention. While any known covariate shift can readily be incorporated by adjusting the weights, unknown shifts can produce misleading results.




\subsubsection*{Acknowledgments}
This research was supported by the Wallenberg AI, Autonomous Systems and Software Program (WASP) funded by Knut and Alice Wallenberg Foundation.

\bibliographystyle{plainnat} 
\bibliography{library}

\end{document}